\documentclass[final]{cvpr}

\usepackage{times}
\usepackage{epsfig}
\usepackage{graphicx}
\usepackage{amsmath}
\usepackage{amssymb}

\usepackage{cuted}
\usepackage{caption}
\usepackage{multirow} 
\usepackage{bbding} 
\usepackage{makecell}
\usepackage{diagbox}

\usepackage{amsthm}

\usepackage{appendix}

\newtheorem{theorem}{Theorem}

\usepackage[pagebackref=true,breaklinks=true,colorlinks,bookmarks=false]{hyperref}

\def\cvprPaperID{2543} 
\def\confYear{CVPR 2021}

\begin{document}

\title{A Decomposition Model for Stereo Matching}

\author{Chengtang Yao, Yunde Jia, Huijun Di\thanks{Corresponding author}, Pengxiang Li, Yuwei Wu\\
Beijing Laboratory of Intelligent Information Technology\\
School of Computer Science, Beijing Institute of Technology, Beijing, China\\
{\tt\small \{yao.c.t,jiayunde,ajon,lipengxiang,wuyuwei\}@bit.edu.cn}
}

\maketitle

\pagestyle{empty}  
\thispagestyle{empty} 

\begin{abstract}
In this paper, we present a decomposition model for stereo matching to solve the problem of excessive growth in computational cost (time and memory cost) as the resolution increases. In order to reduce the huge cost of stereo matching at the original resolution, our model only runs dense matching at a very low resolution and uses sparse matching at different higher resolutions to recover the disparity of lost details scale-by-scale. After the decomposition of stereo matching, our model iteratively fuses the sparse and dense disparity maps from adjacent scales with an occlusion-aware mask. A refinement network is also applied to improving the fusion result. Compared with high-performance methods like PSMNet and GANet, our method achieves $10-100\times$ speed increase while obtaining comparable disparity estimation results.
\end{abstract}

\section{Introduction}
Stereo matching aims to estimate the disparity from a pair of images. It has various downstream applications, such as 3D reconstruction, AR, autonomous driving, robot navigation, etc. Despite years of research on stereo matching, many state-of-the-art methods still face the problem of excessive growth in computational cost and memory consumption as the resolution increases. This problem limits the ability of existing methods to process high-resolution images, and restricts the use of stereo matching methods in practical situations with memory/speed constraints.
\begin{figure}[ht]
	\centering
	\includegraphics[width=.43\textwidth]{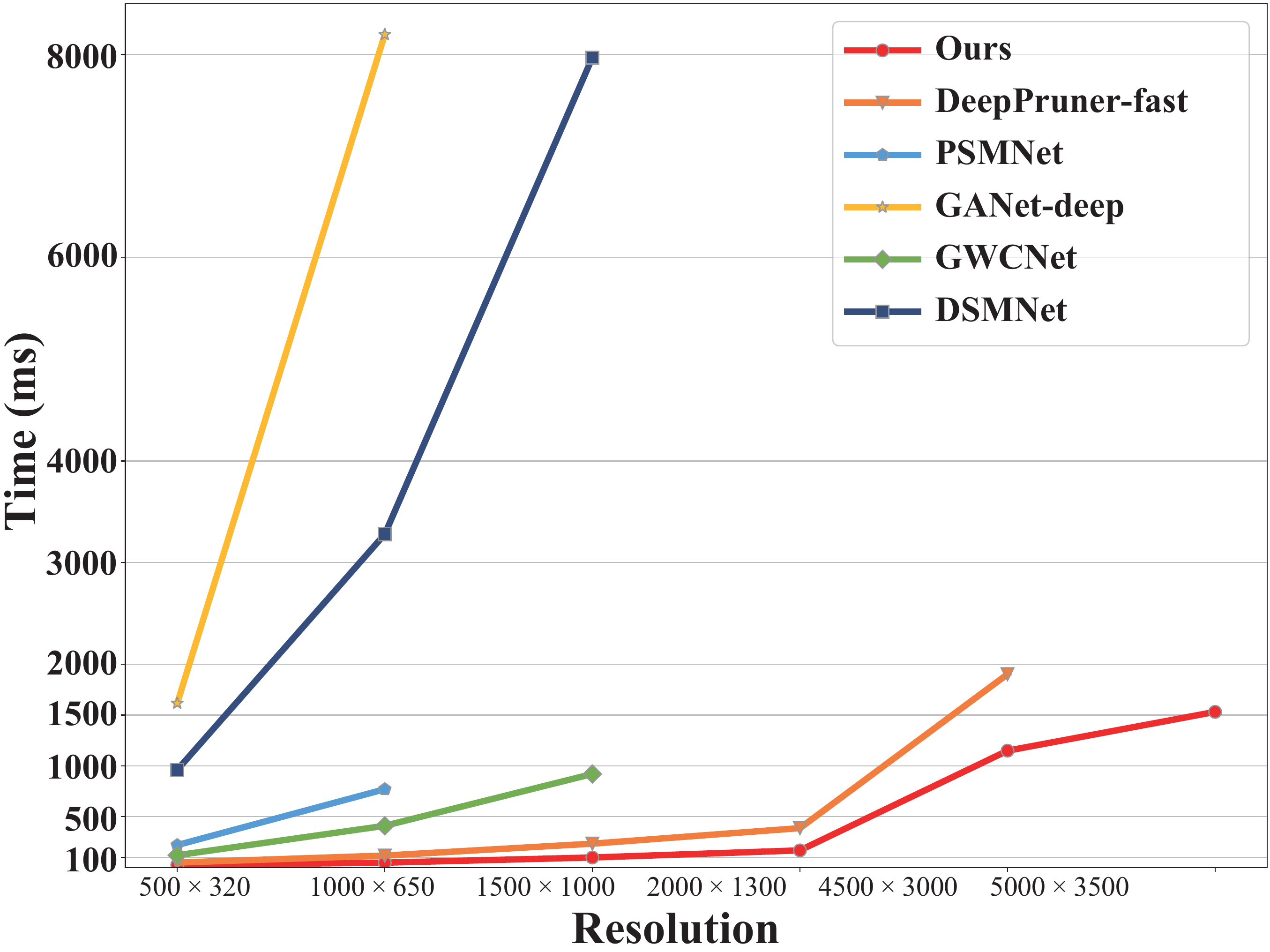}
	\centering
    \caption{As the resolution increases, the growth in time cost of state-of-the-art methods on one 1080 Ti GPU with 11GB of memory. The stopped growth of some curves is because the corresponding method cannot run at the expected resolution on the GPU. Compared to GANet \cite{zhang2019ga}, our model is 100 times faster. Compared to PSMNet \cite{chang2018pyramid}, our model achieves almost $15\times$ speed increase. Compared to DeepPruner \cite{duggal2019deeppruner}, our model achieves almost twice the running speed and lower memory consumption.}
    \vspace{-0.4cm}
    \label{Fig:speed comparison}
\end{figure}

In this paper, we propose a decomposition model for stereo matching. Compared with the excessive growth of many state-of-the-art methods, our model reduces the growth rate by several orders of magnitude as shown in Figure \ref{Fig:speed comparison}. The design of our model is inspired by the following two observations:

(1) It is not necessary to estimate the disparity of all pixels at the highest resolution, such as the disparities on the wall and the floor. As long as the content is not significantly lost during downsampling, the disparity of most areas can be efficiently estimated at low resolution, and then refined at high resolution.

(2) It only needs to consider the disparity estimation of some image details that are lost during downsampling. Fortunately, those lost details are sparse, and their stereo matching is also sparse (i.e., the lost details in the left image mostly only match the lost details in the right image). The sparse matching means less time and memory cost compared to the dense matching.

Based on the first observation, our model only runs dense matching at a very low resolution (such as $20 \times 36$, called reference resolution), ensuring disparity estimation for most regions that are not lost during downsampling. Based on the second observation, our model uses a series of sparse matching, each at a suitable higher resolution, to recover the disparity of lost details scale-by-scale. By decomposing the original stereo matching into a dense matching at the lowest resolution and a series of sparse matching at higher resolutions, the huge cost of original stereo matching can be significantly reduced.

The specific pipeline of our model is shown in Figure \ref{Fig:model}. Our model uses a full cost volume and a cost regularization for dense matching at the reference resolution. From the reference resolution, a series of operations are performed scale-by-scale, until the original input resolution is reached. These operations include four modules: detail loss detection, sparse matching, disparity upsampling, and disparity fusion. The corresponding implementation of the four modules are as follows: (1) In the detail loss detection module, the lost image details are learned unsupervised based on the square difference between deep features from adjacent scales. (2) In the sparse matching module, the sparse disparity map is estimated via cross-correlation and soft-max under the guidance of detected lost details. (3) In the disparity upsampling module, the estimated disparity map from the previous scale is upsampled to the resolution at the current scale via content-aware weights. (4) In the disparity fusion module, the results of the disparity upsampling module and sparse matching module are fused via an occlusion-aware soft mask. A refinement network is also used in this module to improve the fused disparity map.

We analyze the complexity of stereo matching in this paper. For convenience, we define the complexity of matching as the size of the search space \cite{lu2013patch}. We prove that the complexity of original dense stereo matching grows cubically as the input resolution increases, while the complexity of sparse matching in our model only grows logarithmically. In our model, the complexity of dense matching at the reference resolution is fixed, and independent of the input resolution. At the same time, the three other operations, i.e., detail loss detection, disparity upsampling, and disparity fusion, could be efficiently implemented.

In the experiment, we compare our model with state-of-the-art methods over the growth in computation cost and memory consumption. The results show that our model reduces the growth rate by several orders of magnitude. We also compare our model with state-of-the-art methods on Scene Flow dataset \cite{mayer2016large}, KITTI 2015 dataset \cite{menze2015object,Menze2015ISA,Menze2018JPRS}, and Middlebury dataset \cite{scharstein2014high}. The results show that our model is comparable to or even better than state-of-the-art methods with much faster running time and much lower memory cost.

\section{Related Work}
Stereo Matching has been studied for decades \cite{wheatstone1838xviii,marr1979computational,scharstein2002taxonomy}. In the early stage, researchers focus on the analysis of binocular vision and the building of its computation framework \cite{wheatstone1838xviii,ogle1950researches,julesz1971foundations,marr1979computational,von1990binocular}. Later, a series of traditional stereo matching approaches are proposed to improve the framework, including local model \cite{birchfield1999depth,hirschmuller2002real,yoon2006adaptive}, global model \cite{boykov2001fast,kolmogorov2001computing,sun2003stereo} and semi-global model \cite{hirschmuller2008stereo,sinha2014efficient}. Recently, deep-learning-based approaches have emerged and play the most important role in stereo matching \cite{zbontar2015computing,mayer2016large,kendall2017end,chang2018pyramid}. Although both traditional methods and deep learning methods have achieved great performance, they still suffer from the problem of excessive growth in computation cost as the resolution increases.

\textbf{Traditional Methods.}  In order to solve the problem, researchers propose many approaches to reduce the size of search space, by either improving the operations in dense matching, or turning for the sparse-to-dense methods. In the first perspective, researchers propose to reduce the complexity of cost aggregation in terms of the size of image \cite{min2008cost,falkenhagen1997hierarchical}, matching window \cite{richardt2010real}, or disparity space \cite{min2011revisit,falkenhagen1997hierarchical}. The semi-global matching (SGM) \cite{hirschmuller2005accurate} is also proposed to approximate the global energy function with pathwise optimizations from all directions. Different from them improving the operations on dense matching, we replace the dense matching at high resolution with sparse matching to reduce the complexity. In the second perspective, researchers mainly focus on the sparse nature in stereo matching \cite{geiger2010efficient,bleyer2011patchmatch,he2012computing,lu2013patch,sinha2014efficient}. They propose to compute the sparse disparity map over the extracted key points, and then infer the dense disparity map based on the sparse result, like efficient large-scale stereo \cite{geiger2010efficient} and local plane sweeps \cite{sinha2014efficient}. PatchMatch-based methods \cite{bleyer2011patchmatch,he2012computing,lu2013patch} are similar to them but mostly generate the sparse result via random initialization. They are based on the assumption that at least one pixel of the areas is initialized with a label close to the ground truth. Different from them only using sparse matching with local information, we use dense matching at the lowest resolution to provide global information and use sparse matching to recover the local details at high resolution.

\textbf{Deep Methods} In deep learning methods, researchers try to solve the excessive growth in computation cost and memory consumption from two aspects, light-weight network, and improved the computaion with cost volume. In the first aspect, most researchers replace the expensive operations on cost volume with a specially designed module \cite{zhang2019ga,yao2019recurrent,xu2020aanet,liu2020novel}. Some others also improve the whole architecture of network to achieve less computation cost \cite{khamis2018stereonet,yang2019hierarchical,tonioni2019real}. However, the above methods ignore the influence of cost volume. The cost volume will result in cubical growth of computation cost and memory consumption when resolution increasing. In order to further resolve the problem, researchers propose to improve the computaion with cost volume based on the assumption where most content of the cost volume is redundant. They mainly improve the computaion with cost volume via narrowing the disparity space based on the initial estimation and then upsample the result from coarse to fine \cite{yin2019hierarchical,duggal2019deeppruner,chen2019point,yu2020fast,cheng2020deep,gu2020cascade}. Among them, DeepPruner \cite{duggal2019deeppruner} achieves great performance. They use the minimum and maximum disparity regressed by CNN to sample fixed size candidate disparities for the matching at a higher resolution. They also design a light-weight architecture for more efficient cost aggregation and regression. However, they ignore the computation of details at high resolution which are lost in low-resolution matching and are difficult to be recovered with a small size of sampling. Different from the above coarse-to-fine methods, our model is coarse+fine. Instead of only depending on the initial estimation at coarse scale to generate the details at fine scale, we preserve the details in the decomposition of stereo matching. Thus, our model could perform the dense matching at a very low resolution to reduce the search space size but without significant information loss. Furthermore, we resolve the problem from a new view and propose a new pipeline for stereo matching where all the above methods could be integrated into each step in our model to build a more powerful stereo matching, which is a good supplement to the current research mainstream.

\begin{figure*}[ht]
	\centering
	\includegraphics[width=.88\textwidth]{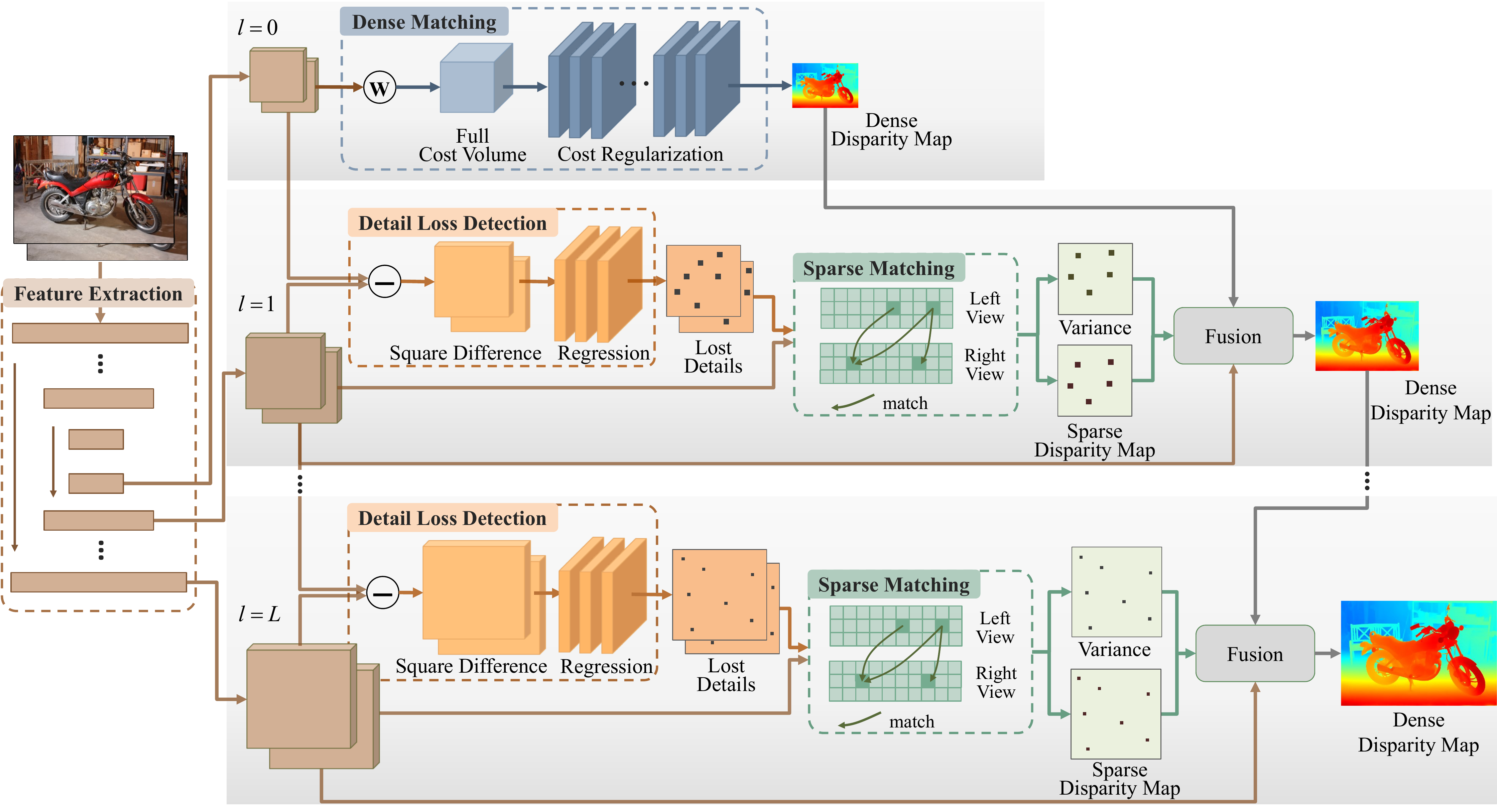}
	\centering
    \caption{Overview of our model. Given a pair of images, we first extract the feature map for the matching at level $l$. We then run the dense matching at the lowest resolution $l=0$ and use sparse matching at different higher resolutions $l \geq 1$. For the computed dense and sparse disparity maps, we fuse them hierarchically to recover the dense disparity map at the original resolution $l=L$. The fusion is composed of disparity upsampling, disparity fusion, and refinement. \textcircled{w} is the warping operation between left and right view. \textcircled{-} represents the computation of the square difference between feature maps from adjacent levels. For details please refer to our method part.}
    \label{Fig:model}
\end{figure*}

\section{Method}
\subsection{Multi-Scale Stereo Matching}
Stereo matching is a dense correspondence problem. It is typically modeled as an exhaustive search process between different areas from left and right views with multi-scale analysis which is used to reduce the ambiguity caused by ill-posed areas. In order to better model this process, we recognize the image as a set of areas $\{A_{l}\}_{l=0}^{l=L}$ where $A_{l}$ represents all areas at each scale/level $l$. The exhausting search process is then modeled as
\begin{equation}
\begin{aligned}
    D_{L} &= \mathcal{F}(\acute{A}_{L}, \grave{A}_{L}, D_{L-1}), \\
    &\hdots \\
    D_{1} &= \mathcal{F}(\acute{A}_{1}, \grave{A}_{1}, D_{0}), \\
    D_{0} &= \mathcal{F}(\acute{A}_{0}, \grave{A}_{0}), \\
    \tilde{D} &= \phi(D_{L},\cdots,D_{0}).
\end{aligned}
\label{EQ:exhaustive search process}
\end{equation}
$\acute{A}_{l}$ and $\grave{A}_{l}$ represent the areas of image from left and right views respectively. $D_{l}$ is the dense disparity maps estimated at level $l$ or is the cost volume taken as the input to the next level $l+1$. $\mathcal{F}(\cdot)$ represents the full matching operation. $\mathcal{F}(\cdot, D_{l})$ represents the full matching operation based on $D_l$. Some methods also contain $\phi(\cdot)$ which represents the fusion of dense disparity maps at different levels.

The $\mathcal{F}(\cdot)$, however, has a high complexity. Given $A_{l}$ with resolution of $\mathrm{H}_{l} \times \mathrm{W}_{l}$ and disparity space size of $\mathrm{D}_{l}$, we define the complexity $\mathrm{O}$ of $\mathcal{F}(\cdot)$ at level $l$ as the size of search space :
\begin{equation}
\begin{aligned}
    \mathrm{O}_{l} &= \mathrm{W}_{l}  \mathrm{H}_{l}  \mathrm{D}_{l}.
\end{aligned}
\label{EQ:complexity of dense matching}
\end{equation}
Then, the whole complexity of exhaustive search process is
\begin{equation}
\begin{aligned}
    \mathrm{O} = \sum_{l=0}^{l=L} \mathrm{O}_{l},
\end{aligned}
\label{EQ:complexity}
\end{equation}
After rewriting Eq. \ref{EQ:complexity}, we obtain the following theorem:
\begin{theorem}
Supposing $s \in \{2,3,\cdots\}$ is the size of upsampling ratio between adjacent levels, $1 < C \leq 8/7$ is a constant value and $\mathcal{O}(\cdot)$ represents the tight upper bound, then the complexity $\mathrm{O}$ of exhaustive search process is
\begin{equation}
\begin{aligned}
    \mathrm{O} &= \mathrm{W}_{0}  \mathrm{H}_{0}  \mathrm{D}_{0} \ \mathcal{O}(s^{3L}C).
\end{aligned}
\label{EQ:complexity of exhaustive search}
\end{equation}
\label{theorem:complexity of exhaustive search}
\vspace{-0.5cm}
\end{theorem}
The theorem \ref{theorem:complexity of exhaustive search} reveals the cubic growth of complexity when using exhaustive search on high resolution images. For the specific proof, please refer to the supplementary materials.

\subsection{Decomposition Model}
As aforementioned, as long as the content is not lost significantly during downsampling, the disparity of most areas, called coarse-grained areas by us, can be estimated efficiently at low resolution and then refined at high resolution. For the image details that are lost during downsampling, we call them fine-grained areas and their disparities should be estimated at high resolution. Therefore, the image areas $A_{l}$ at level $l$ is decomposed as
\begin{equation}
\begin{aligned}
    A_{l} &= \mathrm{CA}_{l} \cup \mathrm{FA}_{l}, \\
    \varnothing &= \mathrm{CA}_{l} \cap \mathrm{FA}_{l},
\end{aligned}
\label{EQ:areas composition}
\end{equation}
where $\mathrm{CA}_{l}$ and $\mathrm{FA}_{l}$ represents the coarse-grained areas and fine-grained areas at level $l$ respectively.

As the stereo matching on coarse-grained areas and fine-grained areas are suitable to be carried out at low and high resolution respectively, we decompose the original stereo matching into a full matching at the lowest level and a series of sparse matching at the rest levels, as shown in Figure \ref{Fig:model}. Our model can be formulated as
\begin{equation}
\begin{aligned}
    \hat{D}_{L} &= \widehat{\mathcal{F}}(\acute{\mathrm{FA}}_{L}, \grave{\mathrm{FA}}_{L}), \\
    &\vdots \\
    \hat{D}_{1} &= \widehat{\mathcal{F}}(\acute{\mathrm{FA}}_{1}, \grave{\mathrm{FA}}_{1}), \\
    D_{0} &= \mathcal{F}(\acute{\mathrm{A}}_{0}, \grave{\mathrm{A}}_{0}),
\end{aligned}
\label{EQ:decomposition form of stereo matching}
\end{equation}
\begin{equation}
\begin{aligned}
    \tilde{D} &= \hat{D}_{L} \cup \cdots \cup \hat{D}_{1} \cup D_{0},
\end{aligned}
\label{EQ:decomposition form of disparity map}
\end{equation}
where $\widehat{\mathcal{F}}(\cdot)$ represents the sparse matching operation, $\hat{D}_{l}$ is the sparse disparity map estimated at level $l$, and $\cup$ means the disparity fusion after the disparity upsampling. The fine-grained areas $\mathrm{FA}_{l}$ will be detected by the detail loss detection module in our model.

The complexity $\hat{\mathrm{O}}$ of $\widehat{\mathcal{F}}(\cdot)$ is calculated as
\begin{equation}
\begin{aligned}
    \hat{\mathrm{O}}_{l} &= \mathrm{W}_{l}  \mathrm{H}_{l}  \mathrm{D}_{l} \  r_{spa,l} \  r_{dis,l},
\end{aligned}
\label{EQ:complexity of sparse matching-raw}
\end{equation}
where $r_{spa,l}$ is the percentage of the sparse details in the left view that are lost during downsampling, and $r_{dis,l}$ is the percentage of the sparse disparity search space relative to the size $\mathrm{D}_{l}$ of the full disparity search space in right view. Without loss of generality, we use $r_{spa,l}$ to approximate $r_{dis,l}$. It is because the stereo matching of the details in the left view is only searched on the details in the right view. And $r_{spa,l}$ represents the percentage of detail pixels in a row, which can approximate the average percentage $r_{dis,l}$ of the pixels to be searched in the full disparity search range. Therefore, the complexity $\hat{\mathrm{O}}$ of $\widehat{\mathcal{F}}(\cdot)$ can be rewritten as
\begin{equation}
\begin{aligned}
    \hat{\mathrm{O}}_{l} &= \mathrm{W}_{l}  \mathrm{H}_{l}  \mathrm{D}_{l} \  r_{l}^2, \\
    r_{l} &= r_{spa,l}.
\end{aligned}
\label{EQ:complexity of sparse matching}
\end{equation}

After rewriting Eq. \ref{EQ:complexity} with Eq. \ref{EQ:complexity of sparse matching} and Eq. \ref{EQ:complexity of dense matching}, we obtain the following theorem:
\begin{theorem}
Supposing $s \in \{2,3,\cdots\}$ is the size of upsampling ratio between adjacent levels, $C$ is a constant value and $\mathcal{O}(\cdot)$ represents the tight upper bound, then the complexity $\hat{\mathrm{O}}$ in our model is
\begin{equation}
\begin{aligned}
    \hat{\mathrm{O}} &= \mathrm{W}_{0}  \mathrm{H}_{0} \mathrm{D}_{0} \ \mathcal{O}(LC), \\
    when &\  r_{l} \leq \sqrt{C/s^{3l}}.
\end{aligned}
\vspace{-0.5cm}
\label{EQ:complexity hierarchical decomposition model}
\end{equation}
\label{theorem:complexity of our model}
\end{theorem}

We present a statistical analysis of data to show that the condition in theorem \ref{theorem:complexity of our model} is satisfiable in most time. Specifically, we make statistical analyses according to the content-aware results of our detail loss detection module (see Sec \ref{SEC:Detail Loss Detection}). As shown in Table \ref{tab:ratio}, almost all the data in Scene Flow dataset satisfy the condition $r_{l} \leq \sqrt{C/s^{3l}}$ in each level with a small value of $C$, which shows the condition of theorem \ref{theorem:complexity of our model} holds in most time. Compared to the exponetial growth of the complexity of exhaustive search process shown by the theorem \ref{theorem:complexity of exhaustive search}, the complexity of stereo matching in our model grows only linearly with the number of levels $L$.

\begin{table}[htbp]
  \centering
  \scalebox{.9}[.9]{%
    \begin{tabular}{|c|c|c|c|}
    \hline
    \diagbox{C}{Per(\%)}{Level} & 1     & 2     & \multicolumn{1}{c|}{3} \\
    \hline
    1     & 99.86 & 70.01 & 49.79 \\
    \hline
    2     & 100.00   & 99.27 & 92.02 \\
    \hline
    3     & 100.00   & 100.00   & 98.87 \\
    \hline
    4     & 100.00   & 100.00   & 99.74 \\
    \hline
    5     & 100.00   & 100.00   & 99.88 \\
    \hline
    6     & 100.00   & 100.00   & 100.00 \\
    \hline
    \end{tabular}%
  }
  \caption{The statistical result about the percentage of data in SceneFlow dataset that satisfy the condition $r_{l} \leq \sqrt{C/s^{3l}}$ in each level.}
  \label{tab:ratio}%
\end{table}%

\subsection{Implementation}
As shown in Figure \ref{Fig:model}, we first use U-Net \cite{ronneberger2015u} to obtain deep features $F_{l}$ on each level $l$ for the stereo matching. We then compute the dense disparity map $D_{0}$ based on dense matching at the lowest resolution $l=0$. We also estimate the sparse disparity map $\hat{D}_{l}$ with sparse matching under the guidance of the detected lost details. We fuse $D_{l-1}$ and $\hat{D}_{l}$ to compute the dense disparity map $D_{l}$ as the input to the next level or the output of our model.

\subsubsection{Decomposed Matching}
\label{SEC:Decompositional Stereo Matching}
\label{SEC:Dense Matching}
\textbf{Dense Matching.} At the lowest level, we follow the previous methods and build a full cost volume for disparity regression \cite{kendall2017end,chang2018pyramid}. It takes up little computation resources due to the negligible size of search space. We also use cost regularization to rectify the cost volume before the disparity regression with softmax. The cost regularization is composed of eight 3D convolutions all following a batch normalization layer. For the specific architecture please refer to the supplementary materials.

\label{SEC:Detail Loss Detection}
\textbf{Detail Loss Detection.} We formulate $FA_{l}$ as details that would disappear in the low level. We use a binary mask $M_{FA}$ to represents the positions of lost details, which is computed by a network $\mathcal{F}_{\rm{DLD}}$ based on the square difference between $F_{l}$ and upsampled $F'_{l-1}$:
\begin{equation}
\small
\begin{aligned}
    M_{FA_{l}} &= \mathcal{F}_{\rm{DLD}}((F_{l}-F'_{l-1})^2; \theta),
\end{aligned}
\label{EQ:detail loss detection}
\end{equation}
where $\theta$ is the parameter of network. The network is composed of three convolution operations and a sigmoid function, and its learning is guided by an unsupervised loss. The unsupervised loss $\mathcal{L}^{\rm{DLD}}_{l}$ is designed by maximizing the differences between $F_{l}$ and $F_{l-1}$ on $FA_{l}$ and forcing the sparsity of $FA_{l}$ :
\begin{equation}
\small
\begin{aligned}
    \mathcal{L}^{\rm{DLD}}_{l} &= \mid {FA}_{l} \mid - \alpha \frac{\sum_{(h,w) \in {FA}_{l}} \parallel F_{l}(h,w)-F'_{l-1}(h,w) \parallel_{2}}{\mid {FA}_{l} \mid}.
\end{aligned}
\label{EQ:unsupervised loss for detail loss detection}
\end{equation}
Benefiting from unsupervised learning, we do not need additional data annotation for training.

\textbf{Sparse Matching.} After obtaining the fine-grained areas, our focus turns to how to conduct the sparse matching on the extracted $\{ FA_{l} \}_{l=0}^{l=L}$. It is not suitable to use cost volume representation, as $FA_{l}$ is content-aware whose shape and size are dynamic but not fixed. Instead, we opt for a direct computation of the disparity map. Specifically, we compute the cost via cross-correlation:
\begin{equation}
\begin{aligned}
    C_{l}(h,w,d) &= <\acute{F}_{l}(h,w), \grave{F}_{l}(h,w-d)>,
\end{aligned}
\label{EQ:cost computation}
\end{equation}
where $\acute{F}_{l}$ and $\grave{F}_{l}$ are the deep features from left and right views respectively, $(h,w) \in \acute{FA}_{l}$ and $(h,w-d) \in \grave{FA}_{l}$. We then use softmax to get the probability distribution:
\begin{equation}
\begin{aligned}
    P_{l}(h,w,d) &= \frac{ \mathrm{e}^{C_{l}(h,w,d)-C_{l}^{max}(h,w)} }{ \sum_{d=0}\mathrm{e}^{C_{l}(h,w,d)-C_{l}^{max}(h,w)} }, \\
    C_{l}^{max}(h,w) &= \mathop{\max}_{d} \  C_{l}(h,w,d).
\end{aligned}
\label{EQ:probability distribution}
\end{equation}
And we regress the sparse disparity map over the computed probability distribution as
\begin{equation}
\begin{aligned}
    \hat{D}_{l}(h,w) &= \sum_{d=0} P_{l}(h,w,d) * d. \\
\end{aligned}
\label{EQ:disparity regression}
\end{equation}
Compared to the full matching, the sparse matching not only reduces the number of points on $\acute{A}_{l}$, but also eliminates a lot of redundant matching on $\grave{A}_{l}$. As for the specific equation of back propagation, please refer to the supplementary materials.

\begin{table*}[htbp]
  \centering
  \scalebox{.79}[.79]{%
    \begin{tabular}{ccc|ccc|ccc|c}
    \hline
    \multicolumn{3}{c|}{$\rm Fusion_{l=1}$} & \multicolumn{3}{c|}{$\rm Fusion_{l=2}$} & \multicolumn{3}{c|}{$\rm Fusion_{l=3}$} & \multirow{2}[4]{*}{EPE} \\
\cline{1-9}    \makecell[c]{dynamic \\ upsampling} & \makecell[c]{disparity \\ fusion} & \makecell[c]{refinement} & \makecell[c]{dynamic \\ upsampling} & \makecell[c]{disparity \\ fusion} & \makecell[c]{refinement} & \makecell[c]{dynamic \\ upsampling} & \makecell[c]{disparity \\ fusion} & \makecell[c]{refinement} &  \\
    \hline
    \hline
       -   &    -   &    -   &    -   &    -   &    -   &    -   &   -    &    -   & 2.104 \\
    \Checkmark &    -   &    -   &    -   &    -   &   -    &    -   &    -   &    -   & 2.086 \\
    \Checkmark & \Checkmark &   -    &   -    &    -   &    -   &    -   &   -    &   -    & 2.039 \\
    \Checkmark & \Checkmark & \Checkmark &   -    &    -   &   -    &    -   &   -    &    -   & 1.286 \\
    \hline
    \Checkmark & \Checkmark & \Checkmark & \Checkmark &   -    &     -  &    -   &    -   &    -   & 1.208 \\
    \Checkmark & \Checkmark & \Checkmark & \Checkmark & \Checkmark &    -   &    -   &   -    &    -   & 1.205 \\
    \Checkmark & \Checkmark & \Checkmark & \Checkmark & \Checkmark & \Checkmark &    -   &   -    &    -   & 0.954 \\
    \hline
    \Checkmark & \Checkmark & \Checkmark & \Checkmark & \Checkmark & \Checkmark & \Checkmark &   -    &   -    & 0.896 \\
    \Checkmark & \Checkmark & \Checkmark & \Checkmark & \Checkmark & \Checkmark & \Checkmark & \Checkmark &   -    & 0.893 \\
    \Checkmark & \Checkmark & \Checkmark & \Checkmark & \Checkmark & \Checkmark & \Checkmark & \Checkmark & \Checkmark & 0.842 \\
    \hline
    \end{tabular}%
  }
  \caption{The influence of each component in the Fusion step obatined on the Scene Flow dataset. $\mathrm{Fusion}_{l}$ represents the Fusion step at level $l$.}
  \label{tab:abalation study - arch}%
\end{table*}%

\begin{table}[htbp]
\centering
\scalebox{.8}[.8]{%
    \begin{tabular}{|c|c|c|c|c|}
    \hline
    \diagbox{Level}{EPE}{Type} & \makecell[c]{Dense \\ Result} & \makecell[c]{Fusion \\ Result} & \makecell[c]{Hard \\ Fusion} & \makecell[c]{Soft \\ Fusion} \\
    \hline
    1 & 3.919  & 3.173 &   3.087    &   2.039 \\
    \hline
    2 & 4.965  & 4.134 &   1.624    &   1.205 \\   
    \hline
    3 & 4.245  & 3.918 &   1.010    &   0.893 \\
    \hline
    \end{tabular}
}
  \caption{The illustration of upsampled dense result $D_{l}^{'}$ and fusion result $\bar{D}_{l}$ on the non-occluded fine-grained areas, and the result from hard fusion and soft fusion. Experiments are conducted on the Scene Flow dataset.}
  \label{tab:abalation study - dense&sparse hard&soft}%
\end{table}%

\subsubsection{Fusion}
\textbf{Disparity Upsampling.} After obtaining $D_{0}$ and $\{ \hat{D}_l \}_{l=1}^{l=L}$, we fuse them hierarchically to compute the final output disparity map at the highest level. At each level, we first upsample the dense disparity map from previous level $D_{l-1}$ to current level ${D}'_{l}$ in a content-aware fashion \cite{wang2019carafe}. The content-aware weights are learned through three convolution operations with the input of the left feature map $\acute{A}_{l}$ at the current level and the dense disparity map $D_{l-1}$ from the previous level.

\textbf{Disparity Fusion.} According to Eq. \ref{EQ:decomposition form of disparity map}, we could model the fusion process as the union of a collection of sets, like the superposition of multiple images, which we call as \emph{hard fusion}. However, such a hard fusion performs poorly in practice due to the influence of occlusion. The occlusion results in severe matching ambiguity in sparse matching, while fine-grained areas contain many edges that are easy to occur in occlusion. Therefore, we propose to fuse the sparse disparity map $\hat{D}_{l}$ and the upsampled dense disparity map ${D}'_{l}$ via a learned mask, which we call as \emph{soft fusion}. Specifically, we use a regression network to generate the soft mask $M_l$. The regression network contains three 2D convolution operations and a sigmoid activation function. The input of regression network is the concatenation of left features $\acute{F}_{l}$, upsampled dense disparity map ${D}'_{l}$, sparse disparity map $\hat{D}_{l}$, mask of current fine-grained areas $M_{\acute{FA}_{l}}$ and robustness mask of sparse matching $\hat{V}_{l}$. The pixel-wise robustness is formulated as the variance in sparse matching:
\begin{equation}
\begin{aligned}
    \hat{V}_{l}(h,w) &= \sum_{d=0} P_{l}(h,w,d) * (\hat{D}_{l}(h,w)-d)^2,
\end{aligned}
\label{EQ:variance in sparse matching}
\end{equation}
where $(h,w) \in \acute{FA}_{l}$ and $(h,w-d) \in \grave{FA}_{l}$. The generation of soft mask via regression network $\mathcal{F}_{\rm{REG}}$ is then formulated as
\begin{equation}
\begin{aligned}
    M_{l} = \mathcal{F}_{\rm{REG}}(cat(\acute{F}_{l},{D}'_{l},\hat{D}_{l},M_{\acute{FA}_{l}},\hat{V}_{l}); \  \theta).
\end{aligned}
\label{EQ:soft mask generation}
\end{equation}
$\theta$ is the parameter of regression network and $cat(\cdot)$ is the concatenation operation. After obtaining the soft mask $M_{l}$, we compute the soft fusion of dense and sparse disparity map as
\begin{equation}
\begin{aligned}
    \bar{D}_{l} &= {D}'_{l}(1-M_{l}) + \hat{D}_{l}M_{l}.
\end{aligned}
\label{EQ:variance in sparse matching}
\end{equation}

\textbf{Refinement.} We further propose a refinement network to improve the sub-pixel accuracy of current dense disparity map $\bar{D}_L$. Specifically, we use $\bar{D}_L$ to warp the right feature maps. Then, we concatenate the warped right feature maps $\grave{F}'_{l}$, left feature maps $\acute{F}_{l}$ and current disparity map $\bar{D}_{l}$ to feed into a refinement network $\mathcal{F}_{\rm{REF}}$ :
\begin{equation}
\begin{aligned}
    D_{l} = \bar{D}_{l} + \mathcal{F}_{\rm{REF}}(cat(\grave{F}'_{l},\acute{F}_{l},\bar{D}_{l}); \ \theta).
\end{aligned}
\label{EQ:refinement}
\end{equation}
$\theta$ is the parameter of the refinement network, The refinement network contains seven convolution operations. All convolutions are followed by a relu activation function and a batch normalization level, except the last convolution that is only followed by a batch normalization level.

\begin{figure*}[ht]
	\centering
	\includegraphics[width=.86\textwidth]{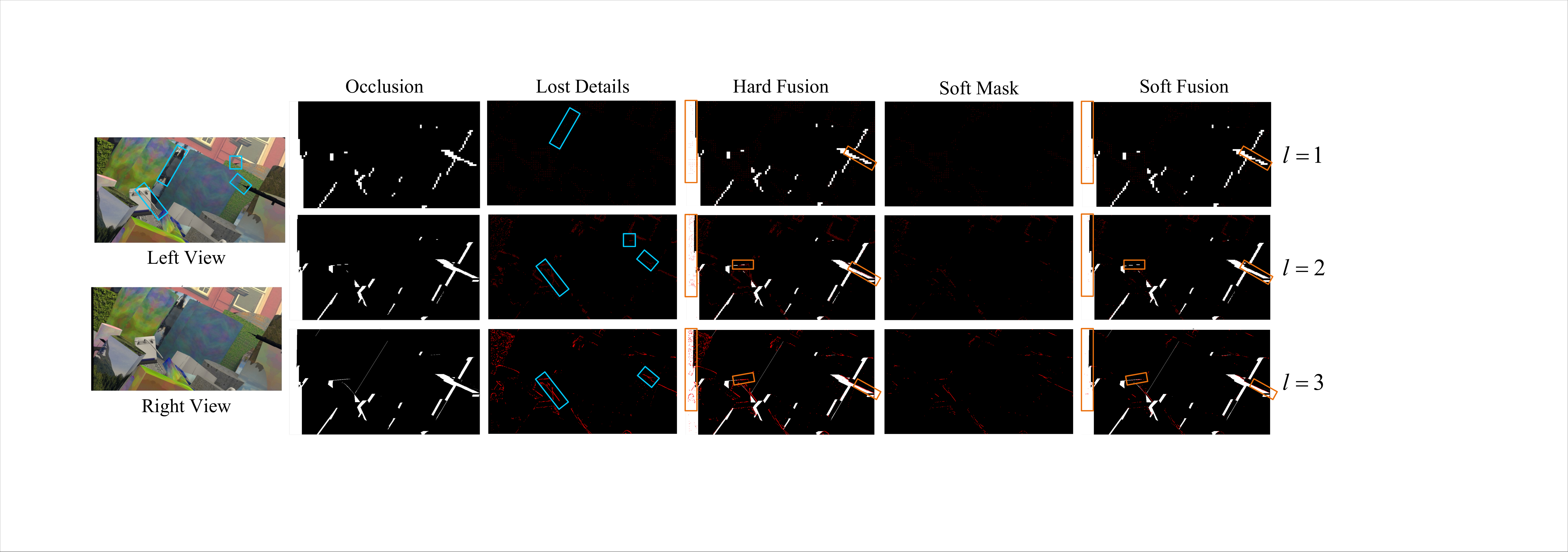}
	\centering
    \caption{The visualization of results from detail loss detection and soft mask. The white areas in the second to last columns represent occlusion. The red points in the third and fourth columns represent the lost details, while the red points in the last two columns represent the learned soft mask.}
    \label{Fig:aba-vis}
\end{figure*}

\section{Experiment}
In this section, we conduct the most analysis of our model based on the Scene Flow dataset \cite{mayer2016large}, except the complexity analysis which is carried out on the high-resolution Middlebury-v3 \cite{scharstein2014high}. We also compare our model with state-of-the-art methods \cite{mayer2016large,kendall2017end,chang2018pyramid,khamis2018stereonet,tonioni2019real,zhang2019ga,yang2019hierarchical,duggal2019deeppruner,badki2020bi3d,xu2020aanet,yang2020waveletstereo} based on Scene Flow \cite{mayer2016large}, KITTI 2015 \cite{menze2015object,Menze2015ISA,Menze2018JPRS}, and Middlebury-v3 \cite{scharstein2014high}.

We train our model end-to-end using Adam optimization with $\beta_1=0.9$, $\beta_2=0.999$ and batch size of $18$ on 3 Nvidia 1080Ti GPUs. We set the maximum of disparity as $216$ and apply color normalization to each input image during training. For the Scene Flow dataset, we train our model for 20 epochs with a learning rate of 0.001 which is then decayed by half every 7 epochs. For KITTI 2015 dataset, we fine-tune the model on mixed KITTI 2012 and KITTI 2015 training sets for 500 epochs. The initial learning rate is set as 0.001 and is decreased by half every 100 epochs after the 200th epoch. As for Middlebury-v3, we fine-tune the model pre-trained on Scene Flow. The learning rate is set to 0.001 for 300 epochs and then changed to 0.0001 for the rest of 600 epochs.

\subsection{Analysis}
\subsubsection{Complexity Analysis}
We analyze the complexity of our model by comparing the growth rate in computation cost with state-of-the-art methods. In order to obtain the curve of growth rate, we resize the image in Middlebury-v3, like the image named Australia, to obtain a sequence of inputs with different resolutions. We then test the running time of each method with the official code. The testing is conducted on a 1080Ti GPU with Cuda synchronization. As shown in Figure \ref{Fig:speed comparison}, most methods stop working on the high-resolution image due to the unaffordable memory consumption, while our method can still run on the image with $5000 \times 3500$ resolution. The time cost of most state-of-the-art methods also grows exponentially when resolution increasing. Different from them, our model has a very low growth rate of time cost benefiting from our decomposition of stereo matching.
\begin{table}[htbp]
  \centering
  \scalebox{.8}[.8]{%
    \begin{tabular}{c||cccc}
    \hline
    Models & EPE   & $\rm > 3px$ & Time (ms) & Mem. (GB) \\
    \hline
    \hline
    GA-Net-15 \cite{zhang2019ga} & 0.84  & -     & $4055^*$ & $6.2^*$ \\
    Bi3D \cite{badki2020bi3d} & 0.73  & -     & $1302^*$ & $10.7^*$ \\
    GCNet \cite{kendall2017end} & 2.51  & 9.34  & 950   & - \\
    PSMNet \cite{chang2018pyramid} & 1.09  & 4.14  & $640^*$  & $6.8^*$ \\
    Waveletstereo \cite{yang2020waveletstereo} & 0.84  & 4.13  & 270   & - \\
    AANet \cite{xu2020aanet} & 0.87  & -     & $147^*$  & $1.2^*$ \\
    \hline
    Deepruner-fast \cite{duggal2019deeppruner} & 0.97  & -     & $87^*$   & $1.9^*$ \\
    DispNetC \cite{mayer2016large} & 1.68  & 9.31  & 60    & - \\
    StereoNet \cite{khamis2018stereonet} & 1.10  & -     & 15    & - \\
    ours  & 0.84  & 3.58  & 50    & 1.6 \\
    \hline
    \end{tabular}%
  }
  \caption{The comparison of algorithms on the Scene Flow dataset. $^*$ represents the result obtained on our machine with official code after Cuda synchronization in a unified setting. $EPE$ is the mean absolute disparity error in pixels. $> 3px$ is the number of pixels whose predicted disparity is deviated from their ground truth by at least 3 pixels.}
  \label{tab: scene flow}%
\end{table}%
Comparing to the high-performing methods, GANet \cite{zhang2019ga}, our model is 100 times faster. Compared to the PSMNet \cite{chang2018pyramid}, our model achieves almost $15\times$ speed increase. As for the DeepPruner \cite{duggal2019deeppruner}, we still achieve almost twice faster running speed and better performance on memory consumption, even current implementation using a fixed-size level of decomposition. According to the theorem \ref{theorem:complexity of our model}, our model could achieve better speed and memory consumption with a dynamic or larger size of decomposition.

\begin{table}[htbp]
  \centering
  \scalebox{.75}[.75]{%
  \renewcommand\tabcolsep{1.0pt} 
    \begin{tabular}{c||ccc|ccc|c}
    \hline
    \multirow{2}[4]{*}{Models} & \multicolumn{3}{c}{Noc} & \multicolumn{3}{c|}{All} & \multirow{2}[4]{*}{Time (ms)} \\
\cline{2-7}          & D1-bg & D1-fg & D1-all & D1-bg & D1-fg & D1-all &  \\
    \hline
    \hline
    GCNet \cite{kendall2017end} & 2.02  & 3.12  & 2.45  & 2.21  & 6.16  & 2.87  & 900 \\
    Bi3D  \cite{badki2020bi3d} & 1.79  & 3.11  & 2.01  & 1.95  & 3.48  & 2.21  & 480 \\
    PSMNet \cite{chang2018pyramid} & 1.71  & 4.31  & 2.14  & 1.86  & 4.62  & 2.32  & 410 \\
    GA-Net-15 \cite{zhang2019ga} & 1.40  & 3.37  & 1.73  & 1.55  & 3.82  & 1.93  & 360 \\
    Waveletstereo \cite{yang2020waveletstereo} & 2.04  & 4.32  & 2.42  & 2.24  & 4.62  & 2.63  & 270 \\
    HSM \cite{yang2019hierarchical} & 1.63  & 3.40  & 1.92  & 1.80  & 3.85  & 2.14  & 140 \\
    HD3S \cite{yin2019hierarchical} & 1.56  & 3.43  & 1.87  & 1.70  & 3.63  & 2.02  & 140 \\
    \hline
    DispNetC \cite{mayer2016large} & 4.11  & 3.72  & 4.05  & 4.32  & 4.41  & 4.34  & 60 \\
    Deepruner-fast \cite{duggal2019deeppruner} & 2.13  & 3.43  & 2.35  & 2.32  & 3.91  & 2.59  & 60 \\
    AANet \cite{xu2020aanet} & 1.80  & 4.93  & 2.32  & 1.99  & 5.39  & 2.55  & 60 \\
    MADNet \cite{tonioni2019real} & 3.45  & 8.41  & 4.27  & 3.75  & 9.2   & 4.66  & 20 \\
    StereoNet \cite{khamis2018stereonet} & -     & -     & -     & 4.3   & 7.45  & 4.83  & 15 \\
    ours  & 1.89  & 3.53  & 2.16  & 2.07  & 3.87  & 2.37  & 50 \\
    \hline
    \end{tabular}%
  }
  \caption{The comparison of algorithms on the KITTI 2015 dataset.  D1 metric measures the percentage of disparity outliers that exceed 3 pixels and $5\%$ of its true value. (Note: the time presented here are not tested in a unified setting. For fair comparison, please refer to the Table \ref{tab: scene flow} and Figure \ref{Fig:speed comparison}).}
  \label{tab: kitti 2015}%
\end{table}%

\begin{figure*}[htbp]
    \centering
    \includegraphics[width=.9\textwidth]{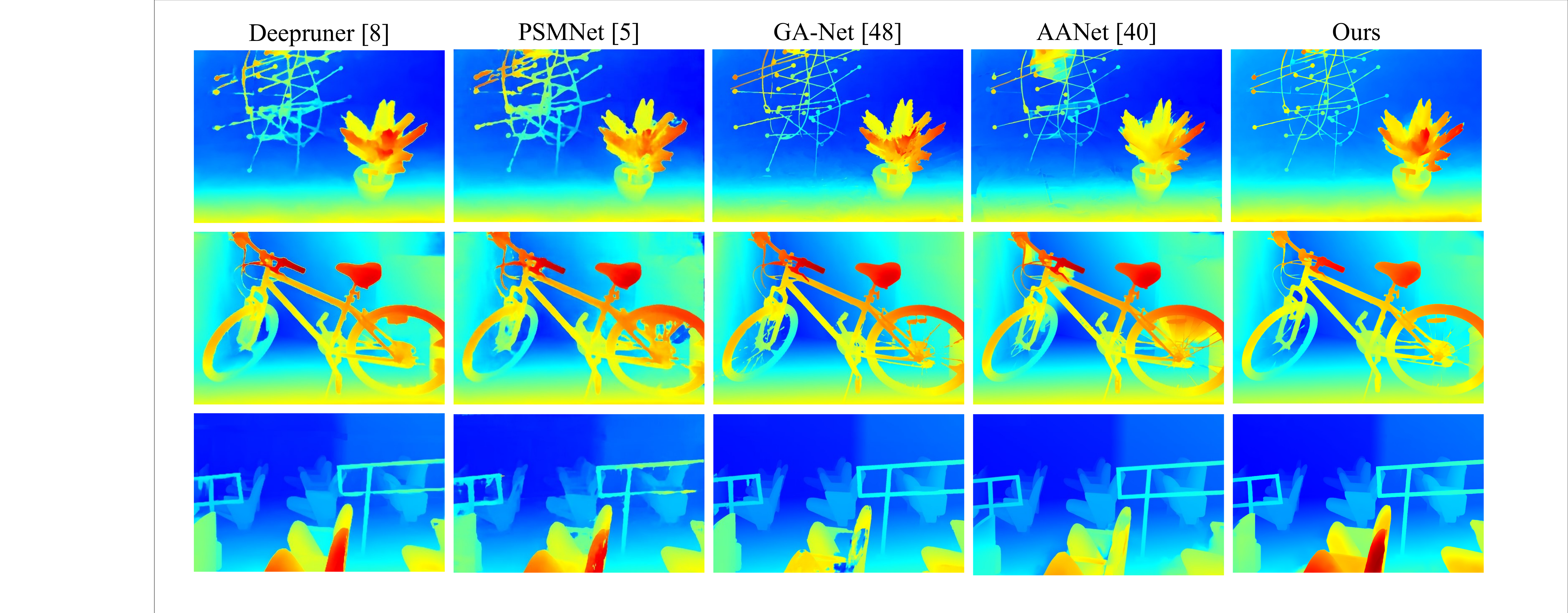}
    \centering
    \caption{The visualization of results on Middlebury-v3.}
    \label{Fig:middlebury}
\end{figure*}

\subsubsection{Ablation Study}
We evaluate our model with different configurations to understand the effectiveness/influence of each component in our model. We use bicubic upsampling to upsample the prediction and compute the end-to-end point (EPE) with ground truth. As illustrated in Table \ref{tab:abalation study - arch}, the error rate is reduced a little after using content-aware upsampling. The error rate is also improved after fusing the sparse results at each level. However, improvement is little. It is because the EPE metric in this table is calculated on the whole spatial space while the sparse result only improves a little area at each level $l$. Thus, we present the differences between upsampled dense result $D_{l}^{'}$ and fusion result $\bar{D}_{l}$ only on the non-occluded fine-grained areas at each level $l$. As shown in Table \ref{tab:abalation study - dense&sparse hard&soft}, the fusion result $\bar{D}_{l}$ is better than the dense result $D_{l}^{'}$, which mirrors the necessity of considering the lost details caused by downsampling. Furthermore, we compare the results from hard fusion and soft fusion to illustrate the necessity of considering the occlusion problem. As shown in Table \ref{tab:abalation study - dense&sparse hard&soft}, the soft fusion is much better than hard fusion at all levels. Furthermore, we analyze the effectiveness of our refinement network. As shown in Table \ref{tab:abalation study - arch}, the performance is significantly improved with the refinement network.

\subsubsection{Visualizing Results of Detail Loss Detection}
The goal of detail loss detection is to find the fine-grained areas at each scale. We visualize the learned binary mask to understand the efficacy of the detector. As shown in Figure \ref{Fig:aba-vis}, the thin or small objects are detected at a corresponding level $l$, especially in the blue boxes of the first and third columns.

\subsubsection{Visualization of Soft Mask}
We compare the visualization of hard fusion, soft mask, and soft fusion to understand the efficacy of soft fusion. As illustrated in Figure \ref{Fig:aba-vis}, most non-robust points, like the occluded points, are eliminated after soft fusion, especially in the orange boxes. We also find that the average value in soft attention is getting smaller when increasing the resolution. We think it is due to the reduced context of sparse matching, which deserves feature exploration, like dense graph net or transformer for better context learning in sparse matching.

\subsection{Benchmark performance}
\textbf{SceneFlow.}
Following previous methods \cite{duggal2019deeppruner,xu2020aanet}, we split the state-of-the-art methods into two parts according to the running time whether exceeds 100ms. It also should be noted that the time and memory cost of some methods are obtained on our 1080Ti GPU with corresponding official code after Cuda synchronization for a fair comparision. As shown in Table \ref{tab: scene flow}, our method achieves comparable results among all the methods. Comparing to the high-performing method, our model reduces the cost of time and memory by order of magnitude while still achieving comparable results. We also provide qualitative results in supplementary materials to show the estimation of our model in different areas, like thin or small objects and large texture-less areas.

\textbf{KITTI 2015.}
In KITTI 2015 dataset, as shown in Table \ref{tab: kitti 2015}, the time cost of reference methods is obtained from the official declaration, which means the different types of GPU and the possibility of Cuda synchronization not being used. As presented in Table \ref{tab: kitti 2015}, our model achieves comparable results and a significant improvement in complexity comparing to the high-performing methods. As for the methods with a running time less than 100ms, we achieve state-of-the-art results. We also provide some visualization in supplementary materials to show the competitive estimation of our model in various scenarios.

\textbf{Middlebury-v3.}
In Middlebury-v3, we estimate the disparities based on the full-resolution inputs when most methods can only afford half resolution. Three examples are visualized in Figure \ref{Fig:middlebury} to show the ability of our model on the high-resolution image. For more results please refer to the supplementary materials.

\section{Conclusion}
In this paper, we have proposed a decomposition model for stereo matching. Compared with many state-of-the-art methods that face the problem of excessive growth in computational cost as the resolution increases, our model can reduce the growth rate by several orders of magnitude. Through a reference implementation of the modules in our model, we achieved state-of-the-art performance among real-time methods, and comparable results among high-performance methods with much lower cost on time and memory. Our model presents a new pipeline for the research of stereo matching, and state-of-the-art methods could be integrated into each step in our model to build a more powerful stereo matching method in the future.

\section{Acknowledgment}
This work was supported by the Natural Science Foundation of China (NSFC) under Grant No.  61773062.

{\small
\bibliographystyle{ieee_fullname}
\bibliography{MyRefs.bib}
}

\clearpage
\appendix
\appendixpage
\addappheadtotoc
\setcounter{table}{0}   
\setcounter{figure}{0}
\setcounter{theorem}{0}

\section{Method Details}
\subsection{Proof of Theorem 1}
\begin{theorem}
Supposing $s \in \{2,3,\cdots\}$ is the size of upsampling ratio between adjacent levels, $1 < C \leq 8/7$ is a constant value and $\mathcal{O}(\cdot)$ represents the tight upper bound, then the complexity $\mathrm{O}$ of exhaustive search process is
\begin{equation}
\begin{aligned}
    \mathrm{O} &= \mathrm{W}_{0}  \mathrm{H}_{0}  \mathrm{D}_{0} \ \mathcal{O}(s^{3L}C).
\end{aligned}
\label{EQ:complexity of exhaustive search}
\end{equation}
\label{theorem:complexity of exhaustive search}
\vspace{-0.5cm}
\end{theorem}

\begin{proof}
As $\mathrm{W}_{l}=\mathrm{W}_{0}s^{l}$ and so does $\mathrm{H}_{l}$ and $\mathrm{D}_{l}$, we get $\mathrm{O}_l=\mathrm{W}_{l} \mathrm{H}_{l} \mathrm{D}_{l} = \mathrm{O}_{0} s^{3l}$.
We then rewrite $\mathrm{O} =\sum_{l=0}^{l=L} \mathrm{O}_{l}$ as  $\mathrm{O} = \sum_{l=0}^{L} \mathrm{O}_{0} s^{3l} = \mathrm{O}_{0} \frac{s^{3(L+1)}-1}{s^{3}-1} < \mathrm{O}_{0} s^{3L} \frac{s^{3}}{s^{3}-1}$. We use $C$ to represent $\frac{s^{3}}{s^{3}-1}$ where $1 < \frac{s^{3}}{s^{3}-1} \le \frac{8}{7}$ because $s$ is at least $2$.
\end{proof}

\begin{figure}[ht]
	\centering
	\includegraphics[width=.44\textwidth]{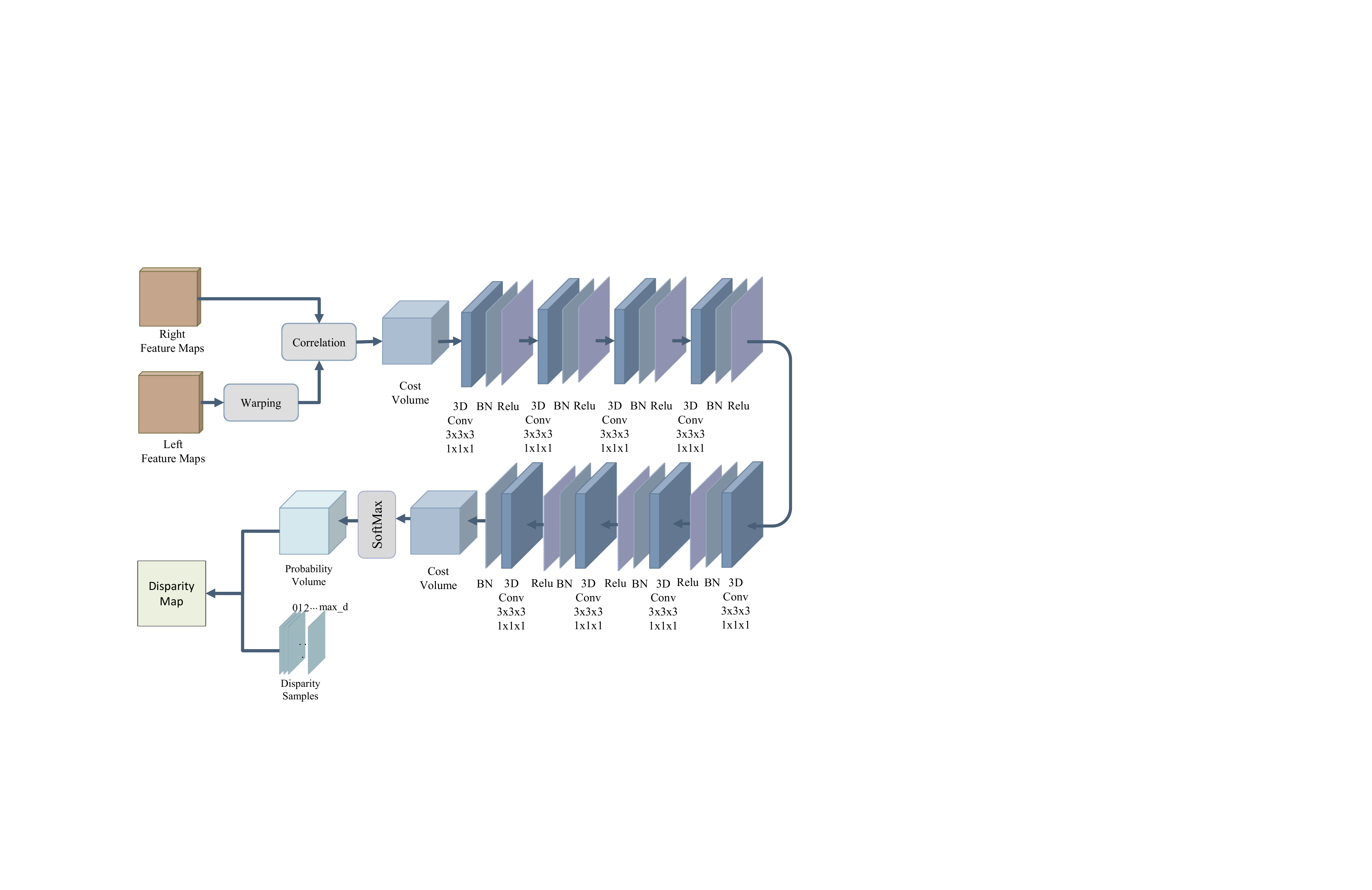}
	\centering
    \caption{Architecture Of dense matching module. $3 \times 3 \times 3$ is the kernel size of 3D convolution, and $1 \times 1 \times 1$ is the stride size.}
    \label{Fig:Dense Matching}
\end{figure}
\subsection{Dense Matching}
As shown in Figure \ref{Fig:Dense Matching}, we build the full cost volume via cross-correlation after warping the right feature maps. We then incorporate eright 3D convolutions to rectify the cost volume. A softmax operation is also used to turn cost volume into a probability volume. The dense disparity map is finally obtained via the regression of probability volume and the sampled disparities.

\begin{table*}[htp]
  \centering
  \scalebox{.85}[.85]{%
  \renewcommand\tabcolsep{2.5pt} 
    \begin{tabular}{c||cccc|cccccc|cccccc}
    \hline
    \multirow{2}[4]{*}{Models} & \multirow{2}[4]{*}{Res} & \multirow{2}[4]{*}{time (s)} & \multirow{2}[4]{*}{time/MP (s)} & \multirow{2}[4]{*}{time/GD (s)} & \multicolumn{6}{c|}{NonOcc}                   & \multicolumn{6}{c}{All} \\
\cline{6-17}          &       &       &       &       & bad 2.0 & bad 4.0 & avgerr & rms   & A90   & A99   & bad 2.0 & bad 4.0 & avgerr & rms   & A90   & A99 \\
    \hline
    \hline
    PSMNet \cite{chang2018pyramid} & Q     & 0.64  & 2.62  & 32.2  & 42.1  & 23.5  & 6.68  & 19.4  & 17.0  & 84.5  & 47.2  & 29.2  & 8.78  & 23.3  & 22.8  & 106 \\
    DeepPruner \cite{duggal2019deeppruner} & \textbf{Q}     & \textbf{0.13}  & 0.41  & 4.38  & 30.1  & 15.9  & 4.80  & 14.7  & 10.4  & 67.7  & 36.4  & 21.9  & 6.56  & 18.0  & 17.9  & 83.7 \\
    GANet \cite{zhang2019ga} & H     & 8.53  & 6.33  & 16.4  & \textbf{18.9} & 11.2  & 12.2  & 35.4  & 40.0  & 84.5  & \textbf{24.9} & \textbf{16.3} & 15.8  & 42.0  & 50.9  & 194 \\
    AANet \cite{xu2020aanet} & H     & 4.56  & 4.17  & 11.0  & 25.2  & 19.6  & 8.88  & 26.2  & 24.2  & 131   & 31.8  & 25.8  & 12.8  & 32.8  & 41.4  & 142 \\
    ours  & \textbf{F}     & \textbf{0.51}  & \textbf{0.10}   & \textbf{0.23}  & \textbf{20.2} & \textbf{11.2} & \textbf{3.72} & \textbf{12.5} & \textbf{10.1} & \textbf{46.8} & \textbf{27.0} & \textbf{17.0} & \textbf{5.37} & \textbf{15.9} & \textbf{15.0} & \textbf{72.2} \\
    \hline
    \end{tabular}%
  }
  \centering
  \caption{The comparison of algorithms on Middlebury-v3 dataset (Q: quadratic resolution, H: half resolution, F: full resolution).}
  \label{tab:middlebury table}%
\end{table*}%

\subsection{Backpropagation of Sparse Matching}
In the main draft, sparse matching is formulated as follows:
\begin{equation}
\begin{aligned}
C_{l}(h, w, d)=<\acute{F}_{l}(h, w), \grave{F}_{l}(h, w-d)>,
\end{aligned}
\label{EQ:cross-correlation}
\end{equation}
\begin{equation}
\begin{aligned}
P_{l}(h, w, d) &=\frac{\mathrm{e}^{C_{l}(h, w, d)-C_{l}^{m a x}(h, w)}}{\sum_{d=0} \mathrm{e}^{C_{l}(h, w, d)-C_{l}^{\max }(h, w)}}, \\
C_{l}^{\max }(h, w) &=\max _{d} C_{l}(h, w, d),
\end{aligned}
\label{EQ:softmax}
\end{equation}
\begin{equation}
\hat{D}_{l}(h, w)=\sum_{d=0} P_{l}(h, w, d) \cdot d.
\label{EQ:regression}
\end{equation}
For the convenience of the derivation of the backpropagation,
we rewrite the above equations as
\begin{small}
\begin{equation}
\hat{D}_{l}(h, w)=\frac{\sum_{d=0} \mathrm{e}^{<} \acute{F}_{l}(h, w), \grave{F}_{l}(h, w-d)>-C_{l}^{\max }(h, w) \cdot d}{\sum_{d=0} \mathrm{e}^{<\acute{F}_{l}(h, w), \grave{F}_{l}(h, w-d)>-C_{i}^{\max }(h, w)}} .
\end{equation}
\end{small}
We then compute the backpropagation over $\acute{F}_{l}(h, w)$ as
\begin{equation}
\begin{aligned}
\frac{\partial \hat{D}_{l}(h, w)}{\partial \hat{F}_{l}(h, w, c)} = \sum_{d=0}(&\hat{F}_{l}(h, w-d, c)(d-\hat{D}_{l}(h, w)\\
&\mathrm{e}^{<\acute{F}_{l}(h, w), \grave{F}_{l}(h, w-d)>-C_{i}^{\max }(h, w)}) \\
/\sum_{d=0} &\mathrm{e}^{<\acute{F}_{l}(h, w), \grave{F}_{l}(h, w-d)>-C_{i}^{\max }(h, w)},
\end{aligned}
\end{equation}
\begin{equation}
\frac{\partial \mathcal{L}}{\partial \dot{F}_{l}(h, w, c)}=\frac{\partial \mathcal{L}}{\partial \hat{D}_{l}(h, w)} \frac{\partial \hat{D}_{l}(h, w)}{\partial \dot{F}_{l}(h, w, c)}.
\end{equation}
As for $\grave{F}_{l}(h, w)$, we compute its backpropagation as
\begin{small}
\begin{equation}
\begin{aligned}
\frac{\partial \hat{D}_{l}(h, w+d^{\prime})}{\partial \grave{F}_{l}(h, w, c)} = &(\acute{F}_{l}(h, w+d^{\prime}, c)(d^{\prime}-\hat{D}_{l}(h, w+d^{\prime})\\
&\mathrm{e}^{<\acute{F}_{l}(h, w+d^{\prime}), \grave{F}_{l}(h, w)>-C_{i}^{\max }(h, w+d^{\prime})}) \\
&/\sum_{d=0} \mathrm{e}^{<\acute{F}_{l}(h, w+d^{\prime}), \grave{F}_{l}(h, w+d^{\prime}-d)>-C_{i}^{\max }(h, w+d^{\prime})}).
\end{aligned}
\end{equation}
\end{small}
\begin{equation}
\frac{\partial \mathcal{L}}{\partial \grave{F}_{l}(h, w, c)}=\sum_{d^{\prime}=0} \frac{\partial \mathcal{L}}{\partial \hat{D}_{l}\left(h, w+d^{\prime}\right)} \frac{\partial \hat{D}_{l}\left(h, w+d^{\prime}\right)}{\partial \grave{F}_{l}(h, w, c)}
\end{equation}

\subsection{Loss}
In addition to the unsupervised loss $\mathcal{L}^{\rm{DLD}}_{l}$ for detail loss detection, we also design a supervised loss for disparity estimation. As there is only ground truth $GT$ of disparity map at the highest level, we downsample the ground truth to each level $GT_l$. At the lowest level, we use smooth L1 between the predicted dense disparity map and the downsampled ground truth:
\begin{equation}
\begin{aligned}
    \mathcal{L}_{0} &= smooth_{L_1}(\mathrm{D}_{0}-GT_{0}), \\
    smoot\mathrm{H}_{L_1}(\epsilon) &= \left\{
                        \begin{array}{lr}
                        0.5\epsilon^2, & if \mid \epsilon \mid < 1 \\
                        \mid \epsilon \mid -0.5, & otherwise
                        \end{array}
                    \right. .
\end{aligned}
\label{EQ:loss at lowest level}
\end{equation}
At higher levels, there are four intermediate results at each level, including the upsampled dense disparity map from previous level ${D}'_{l}$, the sparse disparity map $\hat{D}_{l}$, the fused disparity map $\bar{D}_{l}$ and the refined disparity map $\mathrm{D}_{l}$. To this end, we use a weighted combination of smooth L1 loss over them:
\begin{equation}
\begin{aligned}
    \mathcal{L}_{l} = &\ \  \gamma_1*smooth_{L_1}(\mathrm{D}_{l}-GT_{l}) \\
                        &+\gamma_2*smooth_{L_1}(\bar{D}_{l}-GT_{l}) \\
                        &+\gamma_3*smooth_{L_1}(\hat{D}_{l}-GT_{l} \  M_{\acute{FA}_{l}}) \\
                        &+\gamma_4 *smooth_{L_1}({D}'_{l}-GT_{l}).
\end{aligned}
\label{EQ:loss at rest level}
\end{equation}

Finally, we train our model using end-to-end learning with following loss function:
\begin{equation}
\begin{aligned}
    \mathcal{L} = \mathcal{L}_{0} \cdot \mathrm{W}_{0} + \sum^{l=L}_{l=1}(\mathcal{L}_{l} \cdot \mathrm{W}_{l} + \mathcal{L}^{\rm{DLD}}_{l} w'_{l}),
\end{aligned}
\label{EQ:loss at rest level}
\end{equation}
where $\mathrm{W}_l$ and $w'_l$ are the loss weight.

\section{More Details on Experiment}
We set $\gamma_{1}=0.5$, $\gamma_{2}=0.2$, $\gamma_{3}=0.2$, $\gamma_{4}=0.1$, and $w_{0}=0.037$, $w_{1}=0.11$, $w_{2}=0.33$, $w_{3}=1$, $w_{1}^{\prime}=0.01$.

\subsection{Middlebury-v3}
We present the comparison of results on the Middleburyv3 dataset \cite{scharstein2014high}. We first give a brief description of the metric. time/MP: time normalized by the number of pixels (sec/megapixels). time/GD: time normalized by the number of disparity hypotheses (sec/(gigapixels*ndisp)). bad xx: percentage of ”bad” pixels whose error is greater than xx. avgerr: average absolute error in pixels. rms: root mean-square disparity error in pixels. Axx: xx-percent error quantile in pixels. As shown in Table \ref{tab:middlebury table}, our model achieves the best speed on time/MP and time/GD. our model also obtains almost the best results on most metrics about accuracy.

\subsection{KITTI 2015}
Despite the comparison with state-of-the-art methods in the main draft, we also give a visualization of our results on the KITTI 2015 dataset \cite{menze2015object,Menze2015ISA,Menze2018JPRS}. As shown in Figure \ref{Fig:vis-kitti}, our model achieves competitive estimations in various scenarios.
\begin{figure*}[htbp]
    \centering
    \includegraphics[width=.99\textwidth]{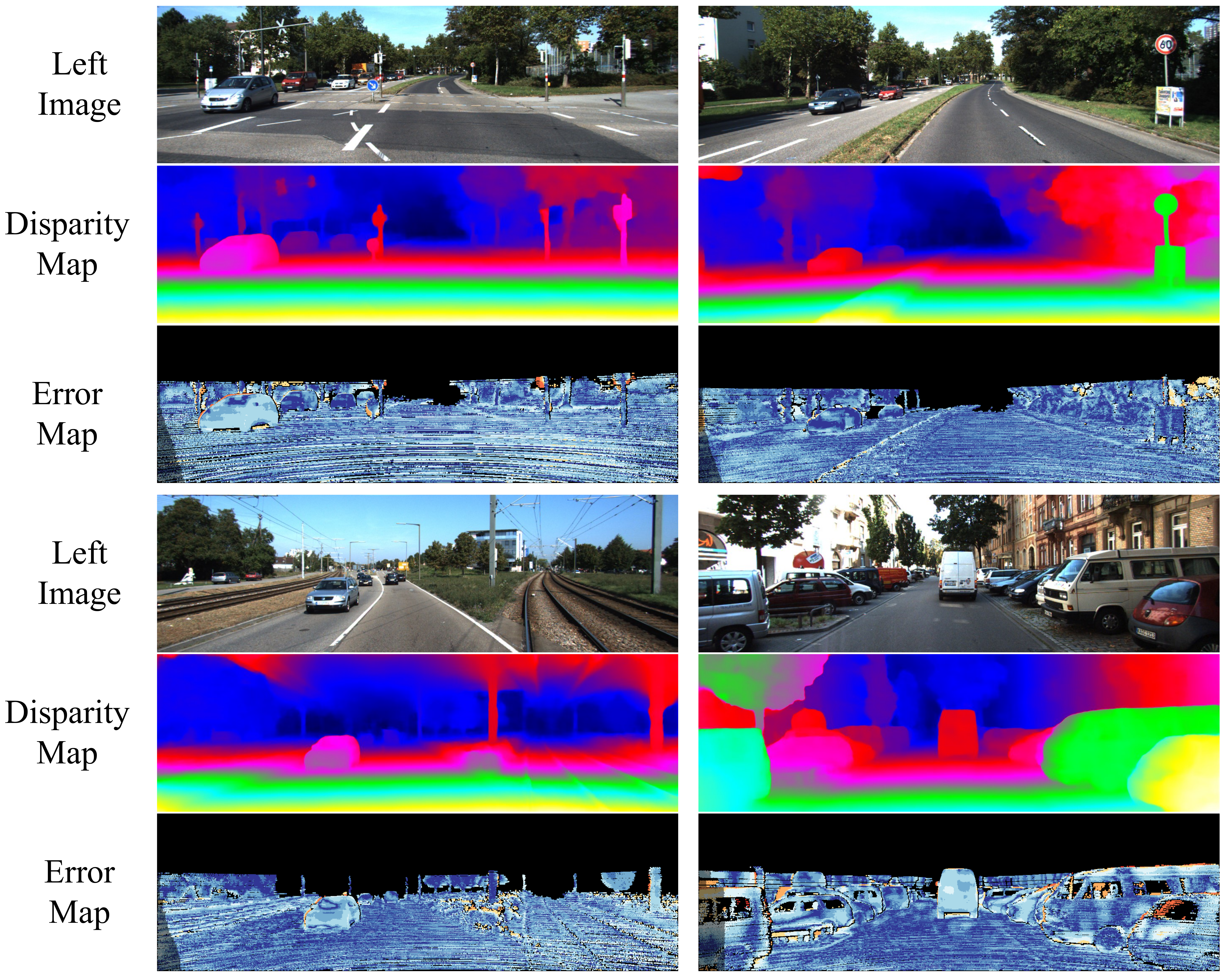}
    \centering
    \caption{Visualization of results on KITTI2015 dataset.}
    \label{Fig:vis-kitti}
\end{figure*}

\subsection{SceneFlow}
We give a visualization of our results on the Scene Flow dataset \cite{mayer2016large}. As shown in Figure \ref{Fig:vis-sceneflow}, our model achieves great results in different areas, like thin or small objects and large texture-less areas.
\begin{figure*}[htbp]
    \centering
    \includegraphics[width=.99\textwidth]{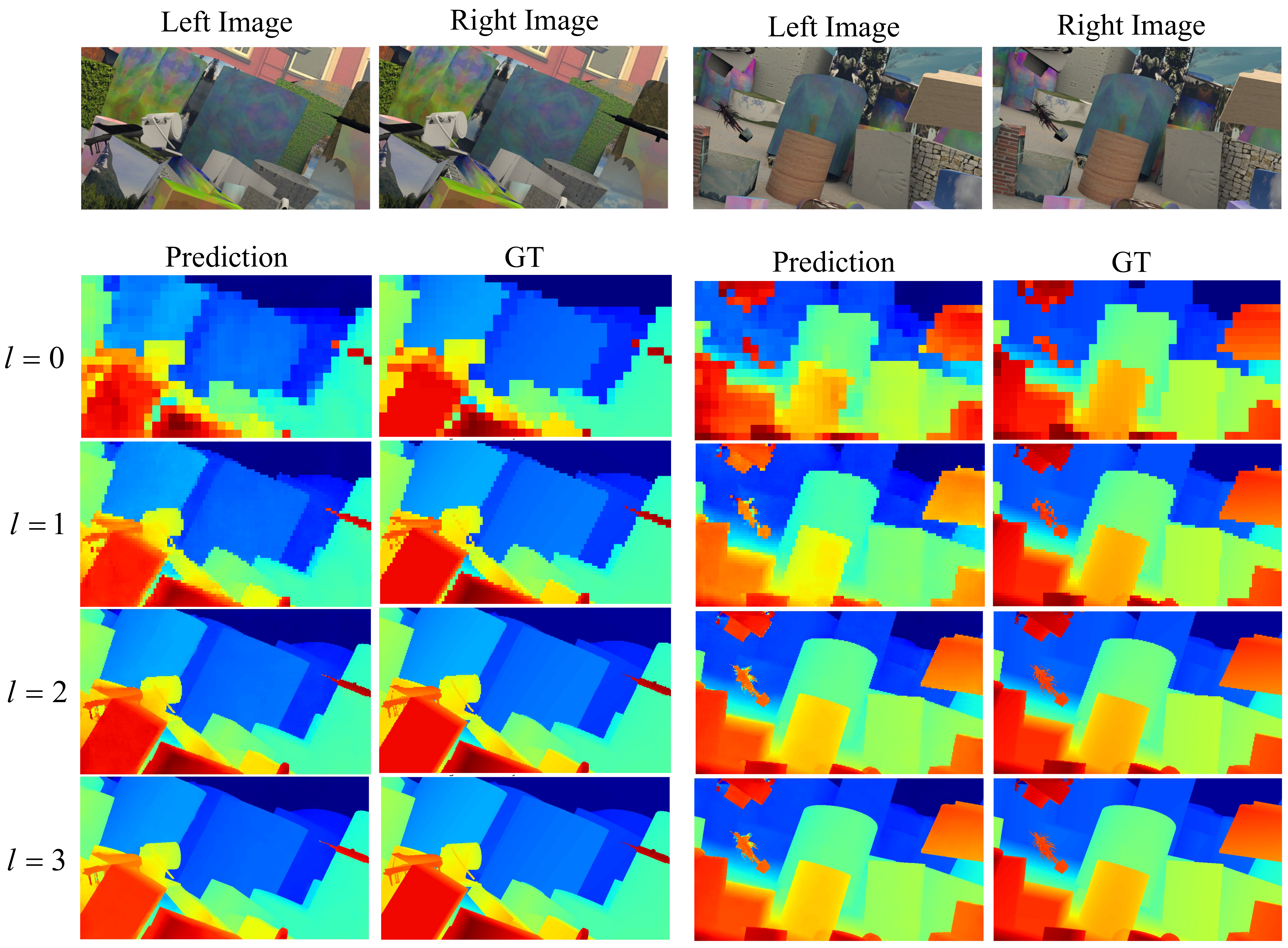}
    \centering
    \caption{Visualization of results on Scene Flow dataset.}
    \label{Fig:vis-sceneflow}
\end{figure*}

\end{document}